\documentclass[11pt]{article}
\def\draft{0}

\title{The Fine-Grained Hardness of\\Sparse Linear Regression}
\usepackage{times}
\usepackage{macros}
\author{
Aparna Gupte\thanks{Email: \texttt{agupte@mit.edu}}\\MIT \and Vinod Vaikuntanathan\thanks{Email: \texttt{vinodv@mit.edu}}\\MIT
}

\begin{document}

\maketitle

\begin{abstract}%
  Sparse linear regression is the well-studied inference problem where one is given a design matrix $\mathbf{A} \in \mathbb{R}^{\slrrows\times \slrcolumns}$ and a response vector $\mathbf{b} \in \mathbb{R}^\slrrows$, and the goal is to find a solution $\mathbf{x} \in \mathbb{R}^{\slrcolumns}$ which is $k$-sparse (that is, it has at most $k$ non-zero coordinates) and minimizes the prediction error $\|\mathbf{A} \mathbf{x} - \mathbf{b}\|_2$. On the one hand, the problem is known to be $\mathcal{NP}$-hard which tells us that no polynomial-time algorithm exists unless $\mathcal{P} = \mathcal{NP}$. On the other hand, the best known algorithms for the problem do a brute-force search among $\slrcolumns^k$ possibilities. In this work, we show that there are {\em no better-than-brute-force} algorithms, assuming any one of a variety of popular conjectures including the weighted $k$-clique conjecture from the area of fine-grained complexity, or the hardness of the closest vector problem from the geometry of numbers. We also show the impossibility of better-than-brute-force algorithms when the prediction error is measured in other $\ell_p$ norms, assuming the strong exponential-time hypothesis.%
\end{abstract}

\section{Introduction}
\label{sec:intro}

High-dimensional inference problems come up naturally in many fields of science and engineering including computational biology and genomics, natural language processing, medical imaging and remote sensing. A well-known example is the standard linear regression problem where one is given data items $\veca_1,\veca_2,\ldots,\veca_{\slrrows} \in \mathbb{R}^\slrcolumns$ and responses $b_1,\ldots,b_{\slrrows} \in \mathbb{R}$ where each $b_i \approx \langle \veca_i, \vecx^*\rangle$ for a hidden parameter vector $\vecx^* \in \mathbb{R}^\slrcolumns$.  When $\slrcolumns \gg \slrrows$, the model is unidentifiable; however, in various application domains (such as the ones above), it makes sense to require that $\vecx^*$ is sparse, that is, it has only a few non-zero coordinates. The resulting problem is called {\em sparse linear regression} which can be formulated as finding
$$ \widehat{\vecx} \in \argmin_{\vecx\in \mathbb{R}^\slrcolumns} \|\vecb-\matA\vecx\|_2 \hspace{.2in}\mbox{ such that } \|\vecx\|_0 \leq k$$
where $\matA \in \mathbb{R}^{\slrrows \times \slrcolumns}$ is the design matrix with $\veca_i$ as its rows, $\vecb \in \mathbb{R}^{\slrrows\times 1}$ is the response vector consisting of the $b_i$'s as entries, and $k$ is the sparsity parameter.

This problem is non-convex and finding an exact as well as approximate solution is known to be $\mathcal{NP}$-hard~\cite{natarajan1995sparse, davis1997adaptive, foster2015variable}.
Despite this negative result, when the design matrix $\matA$ is particularly well-conditioned, polynomial-time algorithms based on $\ell_1$-relaxation, such as basis pursuit and Lasso estimators~\cite{tibshirani,CDS98} as well as the Dantzig selector~\cite{CT07} have been shown to meet the  minimax-optimal lower bound for sparse linear regression. This is also the setting of compressed sensing, which has seen much exciting work in the last two decades.

However, in several high-dimensional regression problems in data science, there is no such guarantee on the design matrix $\matA$, as the features could be correlated. As a result, the problem formulation with general design matrices is also of great interest, and is the main focus of this paper.

On the one hand, the NP-hardness of this problem \cite{natarajan1995sparse, davis1997adaptive, foster2015variable} rules out algorithms that run in time polynomial in $\slrrows, \slrcolumns$ and $k$. On the other hand, the best known algorithm for this problem is trivial brute-force search that runs in time $O(\slrcolumns^{k})$. Concretely, one can enumerate over all possible $\binom{\slrcolumns}{k}$ supports, and run an ordinary least squares regression with an $\slrrows \times k$  design matrix, which results in a $\slrcolumns^k (\slrrows + k)^{O(1)}$ time algorithm \cite{har2016approximate}. \anote{Say that we do not care about the $M$ and $k$ factors.} Improving this to, say, $O(\slrcolumns^{0.999k})$ or even a modest $O(\slrcolumns^{k-1})$ seems out of reach. This state of affairs motivates our question:
\begin{quote}
    Can $k$-sparse linear regression be solved in time $\slrcolumns^{(1-\epsilon)k}$ for {\em some} constant $\epsilon>0$?
\end{quote}
Our main result in this paper shows that the answer is ``no'' assuming popular conjectures from the field of fine-grained algorithms and complexity~\cite{williams2018some}.

First, we show that it is hard to solve $k$-SLR in the $\ell_2$ norm faster than $O(\slrcolumns^{k-\epsilon})$ for any constant $\epsilon>0$, assuming either the min-weight $k$-clique hypothesis~\cite{williams2018some} or the conjectured hardness of the lattice closest vector problem~\cite{bennett2017quantitative}. 
The min-weight $k$-clique problem asks to find a clique with $k$ vertices with the smallest total edge weight, in a weighted undirected graph. The best known algorithm is exhaustive search which runs in time $O(\slrcolumns^k)$ on $N$-node graphs~\cite{williams2018some}. The closest vector problem is a long-studied problem in the geometry of numbers~\cite{Minkowski,MGbook}, which asks to find the closest lattice point to a given target vector, given an arbitrary $n$-dimensional integer lattice. The best known algorithm runs in time $O(2^n)$~\cite{ADS15}, and improving it is a long-standing challenge with considerable impact in the area of combinatorial optimization and lattice-based cryptography~\cite{bennett2017quantitative}.
(For a more formal description of these problems, we refer the reader to Section~\ref{sec:defs}).

Secondly, we show that it is hard to solve $k$-SLR in $\ell_p$ norms (including $p=1$ and odd $p>2$) faster than $O(\slrcolumns^{(1-\epsilon)k})$ under the strong exponential-time hypothesis (SETH)~\cite{impagliazzo2001complexity}. The strong exponential-time hypothesis, formulated by Impagliazzo and Paturi, conjectures that the satisfiability problem (SAT), the cornerstone of computer science, cannot be sped up significantly. More precisely, SETH states that for every $\epsilon>0$, there is an $k\in \mathbb{N}$ such that $k$-SAT on $n$ variables cannot be solved in time $2^{(1-\epsilon)n}$.

These results together settle the exact complexity of the $k$-sparse linear regression problem in the worst-case, assuming popular conjectures in fine-grained complexity and the geometry of numbers. Under these conjectures, there are {\em no better-than-brute-force} algorithms for sparse linear regression.

%Therefore, one resorts to one of the following ways to solve the problem:
%\begin{itemize}
%  \item {\bf Large Runtime:} We know that solving better than $n^{k/2}$ time is hard assuming the hardness of subset sum. Are there better than brute force algorithms?
%  \item {\bf Approximation:} Find an approximate solution.
%  \item {\bf Average-case:} We know that an $\ell_1$ relaxation (such as Lasso) solves the problem well for design matrices that satisfy the restricted eigenvalue condition.
%\end{itemize}
%\vnote{How does the definition of SLR relate to the problem studied by Jordan, Wainwright et al.~\cite{zhang2014lower}?}

%\vnote{Outstanding open question: Generalize the result below to optimal hardness in the $\ell_2$ norm. Going through CVP is unlikely to work, but perhaps one can go directly from SETH?}

\paragraph{A Window into our Techniques: Translating Integrality into Sparsity.} 
%The key technical step in our reductions is a way to translate integrality into sparsity. In a bit more detail,
The hardness of problems such as the closest vector problem are closely connected to integer linear programming, where hardness arises from the fact that one has to come up with {\em integer solutions} to an optimization problem. On the other hand, there is no such integrality condition in $k$-sparse linear regression; rather, the hardness seems to arise from the fact that we insist on {\em sparse} solutions. Thus, a key technical contribution of our work is showing several ways to map two apparently different sources of hardness, namely integrality and sparsity. 

In the rest of the paper, we provide a self-contained exposition of our results. We start with definitions in Section~\ref{sec:defs}, and proceed to show the reduction from the min-weight $k$-clique problem to sparse linear regression in Section~\ref{sec:min-k-clique-slr} and the reduction from the lattice closest vector problem in Section~\ref{sec:cvp-slr}. Finally, in Section~\ref{sec:seth2slr}, we also show that there are no sub-quadratic $\slrcolumns^{2-\epsilon}$-time algorithms for $k$-SLR when $k=2$, assuming the strong exponential-time hypothesis.

\subsection{Related work}\label{sec:related-work}
\paragraph{Lower bounds.} Natarajan~\cite{natarajan1995sparse} showed the $\mathcal{NP}$-hardness of solving sparse linear regression, which rules out $(\slrcolumns k)^c$-time algorithms for any absolute constant $c$. Har-Peled, Indyk and Mahabadi~\cite{har2016approximate} showed the hardness of solving $k$-SLR faster than $O(\slrcolumns^{k/2})$ time, under the well-known $k$-SUM conjecture from fine-grained complexity. However, both results left a gap between the best known upper bound which is brute-force search that runs in $\slrcolumns^k$ time, and a lower bound of $\slrcolumns^{k/2}$. 

\paragraph{Algorithms.}
Har-Peled, Indyk and Mahabadi~\cite{har2016approximate} also give an algorithm for {\em approximate} sparse linear regression {\em with pre-processing} that runs in time slightly better than brute force, namely in time $\tilde{O}(\slrcolumns^{k-1})$. Gupte and Lu~\cite{guptelu2020fine-grained} showed a slight improvement to the standard $k$-SLR problem in the {\em noiseless} setting, namely they showed an $O(\slrcolumns^{k-1})$ time algorithm. They also give a $\slrcolumns^{\Omega(k)}$-time lower bound in the noiseless setting, under the Exponential Time Hypothesis. In contrast to both these works, our result shows that even such a modest improvement is impossible to achieve for the standard and well-studied $k$-SLR problem. That is, we show that $k$-SLR cannot be solved in time $O(\slrcolumns^{k-\epsilon})$ for any constant $\epsilon > 0$, under the popular min-weight $k$-clique conjecture~\cite{williams2018some} from fine-grained complexity.

\paragraph{Compressed Sensing and the LASSO.}
While the focus of this paper is the worst-case complexity of sparse linear regression, its {\em average-case} complexity has been the object of much study. Methods based on $\ell_1$-relaxation, such as basis pursuit and Lasso estimators~\cite{tibshirani,CDS98} as well as the Dantzig selector~\cite{CT07} have been shown to meet the  minimax-optimal lower bound in some cases. In particular, they focus on the setting where $\vecb = \matA \vecx + \vecw$ where the noise vector $\vecw \sim \mathcal{N}(0,\sigma^2)$ comes from a Gaussian distribution. While these methods perform well when the design matrix $\matA$ is well-conditioned in an appropriate sense, Zhang, Wainwright and Jordan~\cite{zhang2014lower} show that such a condition is necessary unless $\mathcal{NP} \subseteq \mathcal{P}/\mathsf{poly}$. Recently, and concurrently with this work, Kelner et al~\cite{kelnerpreconditioning} show random-design instances of sparse linear regression that are hard for a particular class of algorithms, namely preconditioned Lasso. In contrast, we show hardness for {\em worst-case} design matrices, but with respect to {\em any algorithm} that runs faster than trivial brute force search.

\medskip\noindent
Further afield, the areas of fine-grained complexity and the geometry of numbers (or the study of integer lattices) have recently been shown to have interesting connections to statistical problems. For example, Backurs, Indyk and Schmidt~\cite{BIS17} show that certain empirical risk minimization (ERM) problems cannot be solved in sub-quadratic time, under the SETH conjecture. In another result, Backurs and Tzamos~\cite{backurs2017improving} showed that the classical Viterbi algorithm for finding the most likely path in a Hidden Markov Model is in fact optimal under the min-weight $k$-clique hypothesis which we rely on in this work. Bruna, Regev, Song and Tang~\cite{BRST20} recently showed the hardness of problems in robust statistics and learning mixtures of Gaussians, assuming the hardness of lattice problems.

%Finally, 
%A robust version of sparse linear regression, where a fraction of the responses $b_i$ are arbitrarily corrupted by an adversary, was studied by Brennan and Bresler~\cite{BB20}, where they demonstrated a statistical-computational gap for this problem.

%\paragraph{Other Problems.} Berthet and Rigollet~\cite{BR13} show average-case lower bounds for sparse PCA under the hardness of the planted clique problem.  Krauthgamer et al.~\cite{KNV15} also show results along these lines. Ma and Wu~\cite{MaWu15} use the planted clique problem to show hardness of the submatrix detection problem.

% \input{defs.tex}
\section{Preliminaries and definitions}\label{sec:defs}

\paragraph{Notations.} Bold lower-case letters, such as $\vecx$, denote (column) vectors. Bold uppercase letters, such as $\matA$, denote matrices. $\mathbb{R}$ stands for the set of real numbers, $\mathbb{Z}$ for the set of integers, and $\mathbb{R}^+$ (resp. $\mathbb{Z}^+$) for the set of non-negative real numbers (resp. integers). We will let $\|\vecx\|_p$ denote the $\ell_p$ norm of $\vecx \in \mathbb{R}^\slrcolumns$, that is, $\|\vecx\|_p = \big( \sum_{i\in [\slrcolumns]} |x_i|^p \big)^{1/p}$. The Hamming weight of $\vecx$, also called its $\ell_0$ norm,  is simply the number of non-zero coordinates of $\vecx$. A vector $\vecx \in \R^\slrcolumns$ is called $k$-sparse if its Hamming weight is at most $k$, that is, if $\| \vecx\|_0 \leq k$.

\subsection{Sparse linear regression} 

We now define the sparse linear regression problem. 

\begin{definition}[$k$-SLR$_p$]
For any integer $k \ge 2$ and $1 \le p \le \infty$, the $k$-sparse linear regression problem with respect to the $\ell_p$ norm is defined as follows.
Given a matrix $\matA \in \R^{\slrrows \times \slrcolumns}$, a target vector $\vecb \in \R^{\slrrows}$, and a number $\delta > 0$, distinguish between:
\begin{itemize}
    \item a {\em \textbf{YES}} instance, where there is some $k$-sparse $\vecx \in \R^\slrcolumns$ such that $\| \matA \vecx - \vecb \|_p \le \delta$; and 
    \item a {\em \textbf{NO}} instance, where for all $k$-sparse $\vecx \in \R^\slrcolumns$, $\| \matA \vecx - \vecb \|_p > \delta$.
\end{itemize}
\end{definition}

$k$-sparse linear regression can be trivially solved in time $\slrcolumns^{k}$. In this work, we are interested in determining if there are significantly faster algorithms. We will sometimes also give the algorithms more power, by allowing them unbounded time to preprocess the design matrix $\matA$ in an ``offline'' phase, and require them to be fast only in an ``online'' phase where they are given $\vecb$. We call this variant {\em sparse linear regression with preprocessing}.

\subsection{Fine-grained complexity and its hypotheses}\label{sec:assumptions}

Fine-grained complexity is the study of exact runtimes of algorithms for various problems. We present the relevant problems and results here, and refer the reader to the survey by Vassilevska-Williams \cite{williams2018some} for a detailed exposition.

% k-sat
\paragraph{$k$-SAT and the SETH Hypothesis.}
A central problem in fine-grained complexity is the study of the satisfiability problem, and the associated strong exponential-time hypothesis.
Given a CNF-SAT formula on $n$ variables with each clause having size at most $k$, the $k$-SAT problem asks to determine whether the formula is satisfiable or not.

Impagliazzo and Paturi \cite{impagliazzo2001complexity} introduced the strong exponential time hypothesis, conjecturing the fine-grained hardness of $k$-SAT.

\begin{definition}[Strong Exponential Time Hypothesis (SETH)]
For every $\eps > 0$ there exists an integer $k \ge 3$ such that no (randomized) algorithm can solve $k$-SAT on $n$ variables in $2^{(1 - \eps) n}$ time.
\end{definition}

We also use the non-uniform version of SETH to show hardness of sparse linear regression with preprocessing.

\begin{definition}[Non-uniform SETH]
For every $\eps > 0$ there exists an integer $k \ge 3$ such that no family of circuits of size $2^{(1 - \eps)n}$ solves $k$-SAT on $n$ variables.
\end{definition}

\paragraph{The Weighted $k$-Clique Conjecture.}
% k-clique
Suppose we are given an undirected graph $G = (V, E)$ with $|V| = \slrcolumns$ nodes. For a constant $k \ge 3$, a $k$-clique of $G$ is a fully connected subgraph of $G$ of size $k$. The $k$-clique problem asks to distinguish between \textbf{YES} instances, where $G$ contains a $k$-clique, and \textbf{NO} instances, where $G$ does not contain any $k$-clique. Given an {\em edge-weighted} graph and a threshold weight $W$, the Min-Weight $k$-Clique problem asks to distinguish between $\textbf{YES}$ instances  which are graphs $G$ that contain a $k$-Clique of weight at most $W$, and \textbf{NO} instances where $G$ does not contain any $k$-clique of weight at most $W$.  

\begin{definition}[Min-weight-$k$-clique]
Given a graph $G$ with integer edge weights and a weight threshold $W$, determine whether $G$ has a $k$-clique with weight at most $W$ or not.
\end{definition}

The naive algorithm for both problems requires enumerating all $\binom{\slrcolumns}{k}$ subsets of vertices of size $k$, and this takes time $O(\slrcolumns^k)$. In the case of the $k$-clique conjecture, slightly faster algorithms that run in time $O(\slrcolumns^{\omega k/3})$ are known \cite{nevsetvril1985complexity}, where $\omega < 2.373$ is the matrix multiplication constant.
However, no non-trivial algorithms are known for the weighted $k$-clique problem.  It is conjectured that the problem of finding the minimum weight $k$-clique requires time $\slrcolumns^{k - o(1)}$ \cite{williams2018some}.
% \anote{O notation could hide log factors in n}

Under the min-weight $k$-clique hypothesis, Backurs and Tzamos~\cite{backurs2017improving} showed that the classical Viterbi algorithm for finding the most likely path in a Hidden Markov Model is in fact optimal. Under this hypothesis, tight lower bounds are also known for the the local alignment problem \cite{abboud2014consequences} from computational biology and the Maximum Weight Rectangle problem \cite{backurs2016tight} from computational geometry.

\begin{definition}[Min-weight $k$-clique hypothesis]
    The Min-Weight $k$-Clique problem on $\slrcolumns$ node graphs with edge weights in $\{-\slrcolumns^{100k}, \ldots, \slrcolumns^{100k} \}$ requires (randomized) $\slrcolumns^{k - o(1)}$ time.
\end{definition}

% This conjecture can be equivalently stated in terms of the hardness of finding the maximum weight $k$-clique in a graph with non-negative edge weights as follows.

% \begin{conjecture}[Min-weight $k$-clique hypothesis]
%     The problem of finding a Max-weight $k$-clique on a graph with $d$ nodes and edge weights in $\{0, 2d^{100k} + 1 \}$ requires (randomized) $d^{k - o(1)}$ time.
% \end{conjecture}
% The conjecture in \cite{williams2014finding} is stated in terms of a search problem, where negative edge weights are allowed.
% For our purposes, it will be useful to consider the decision problem as defined here. We now show that if there is an algorithm that detects a $k$-clique of weight at most $W$ in time $O(d^{k - \epsilon})$ for some $\epsilon > 0$, then there is an algorithm that finds such a 
% \anote{Search to Decision reduction: Lose factor of $n$? But it seems like the optimization problem is believed to share the same hardness.}

We note that the following variants of the min-weight $k$-clique problem are as hard as the original one. First, we can assume that the edges are non-negative and in the range $\{1, \ldots, 2\slrcolumns^{100k} + 1\}$. To see this, let $w^*$ be the minimum weight of the edges in a graph. Create a new graph $G'$ with the same vertex and edge sets. The edge weights of $G'$ are obtained by adding $|w^*| + 1$ to the weight of the corresponding edge in $G$. Then $G'$ has a $k$-clique of weight at most $W + \binom{k}{2} (|w^*| + 1)$, if and only if $G$ has a $k$-clique of weight at most $W$.

Secondly, we can assume that there is a partition of the vertex set into $k$ subsets such that the $k$-clique has one vertex from each subset.
To see this, suppose there is a $k$-clique $K$ in $G$ with weight at most $W$. Then, we claim that we can assume that the vertex set $V$ is partitioned into $k$ subsets $V = V_1 \cup \ldots, V_k$ such $|K \cap V_i| = 1$ for all $i \in \{1, \ldots, k\}$.
To see this, randomly partition the vertices into $k$ subsets, by assigning each vertex to a subset uniformly and independently at random. Then, for any given $k$-clique (and in particular, the min-weight one), the probability that it has exactly one vertex in each subset is $\prod_{i=0}^{k-1} \left( \frac{k - i}{k} \right) = \frac{k!}{k^k}$, which is at least $\sqrt{2 \pi k} \cdot e^{-k} \cdot e^{\frac{1}{12k + 1}} > e^{-k}$ by Stirling's formula. Repeating this $O(e^k)$ many times completes the argument. Note that although the number of repetitions is exponential in $k$, it is independent of $\slrcolumns$, and so is sufficient for our purposes.

\subsection{The closest vector problem}
% what is a lattice, intro to cvp, defn of dist
A lattice $\mathcal L$ is the set of the integer linear combinations of a set of linearly independent vectors (a basis) $\matB = (\vecb_1, \vecb_2, \ldots, \vecb_{\cvpcolumns})$ where each $ \vecb_i \in \R^{\cvprows}$:
\[\mathcal L = \mathcal L(\matB) = \left\{ \matB \vecz \mid \vecz \in \Z^{\cvpcolumns} \right\}.\]
Given a lattice basis $\matB$, and a target vector $\vect$, define $\dist_p(\mathcal L(\matB), \vect) = \min_{\vecx \in \mathcal L(\matB)} \|\vecx - \vect \|_p$ to be the minimum $\ell_p$ distance from $\vect$ to $\mathcal L = \mathcal L(\matB)$. 

%Roughly speaking, the closest vector problem is the problem of finding a vector $\vecx \in \mathcal{L}(\matB)$ that minimizes .

% definition of CVP

We follow \cite{bennett2017quantitative} for definitions of lattice problems. 

\begin{definition}[$\gamma$-approximate Closest Vector Problem with respect to the $\ell_p$ norm ($\gamma$-$\CVP_p$)]
For any $\gamma \ge 1$ and $1 \le p \le \infty$, given a lattice $\mathcal L$ (specified by its basis $\matB \in \R^{\cvprows \times \cvpcolumns}$), a target vector $\vect \in \R^n$ and a number $r > 0$, distinguish between a \textbf{YES} instance, where $\dist_p(\mathcal L, \vect) \le r$ and a \textbf{NO} instance, where $\dist_p(\mathcal L, \vect) > \gamma r$.
\end{definition}

We refer to the exact version of the problem, when $\gamma = 1$, as $\CVP_p$. 

The best known algorithm for exact CVP$_2$ runs in time $2^{\cvpcolumns + o(\cvpcolumns)}$ time \cite{ADS15}. However, for $p \neq 2$, the best known algorithms for CVP$_p$ are slower, i.e., they run in time $\cvpcolumns^{O(\cvpcolumns)}$ \cite{Kannan83, micciancio2014fast}.
Bennett, Golovnev and Stephens-Davidowitz~\cite{bennett2017quantitative} proved an SETH-based hardness bound for the Closest Vector Problem. In particular, for many $p$, including odd integers $p \ge 1$ and $p = \infty$, they show that $\CVP_p$ cannot be solved in time $2^{(1 - \eps) \cvpcolumns}$ for any $\eps > 0$, unless the Strong Exponential Time Hypothesis (SETH) is false.

% SETH-hardness of CVP
\begin{theorem}[SETH-hardness of CVP$_p$ \cite{bennett2017quantitative}]\label{thm:sethcvp}
For every $\eps > 0$ and every integer $p \ge 1, p \notin 2\Z$ there is no $2^{(1 - \eps)\cvpcolumns}$ algorithm for CVP$_p$, unless SETH is false.
\end{theorem}

CVP with preprocessing (CVPP) refers to algorithms for the closest vector problem which are allowed an unbounded-time offline preprocessing phase where they get the basis $\matB$, and are required to be fast only in an online phase where they receive the target vector $\mathbf{t}$ and the distance threshold $r$. Aggarwal, Bennett, Golovnev and Stephens-Davidowitz show the following hardness results for CVPP$_p$.

\begin{theorem}[\cite{aggarwal2021fine}]
For every $\eps > 0$ and every integer $p \ge 1, p \notin 2\Z$ there is no $2^{(1 - \eps)\cvpcolumns}$ algorithm for CVPP$_p$, unless non-uniform SETH is false. 
% For every $p \ge 1, p \neq 2$, CVPP$_p$ has no $2^{o(n)}$ algorithm unless non-uniform ETH is false.
\end{theorem}

Examining their proof more carefully, we notice that they prove a stronger result. In particular, they show SETH-hardness of the following problem, which we refer to as $(0,1)$-$\CVP_p$.

\begin{definition}[$(0,1)$-$\CVP_p$]
For any $1 \le p \le \infty$, the $(0,1)$-Closest Vector Problem with respect to the $\ell_p$ norm is the following promise problem. Given a lattice $\mathcal L$ specified by a basis $\vec B \in \R^{\cvprows \times \cvpcolumns}$, a target vector $\vec t \in \R^\cvprows$ and two numbers $r, \tau > 0$, distinguish between the following two cases:

\begin{itemize}
    \item \textbf{YES} instances, where there exists some $\vec {y}^* \in \{0, 1\}^\cvpcolumns$ such that $\|\vec{B} \vec{y}^* - \vec {t}\| \le r$; and
    \item \textbf{NO} instances, where for all $\vec y \in \Z^\cvpcolumns$, $\| \vec B \vec y - \vec t\| \ge r + \tau.$
\end{itemize}
\end{definition}

Their method does not show hardness of $\CVP_2$ (or $(0,1)$-$\CVP_2$). On the other hand, to the best of our knowledge, there is no $2^{(1- \epsilon) \cvpcolumns}$-time algorithm for $(0, 1)$-CVP$_2$ for any $\epsilon > 0$. Indeed, as we observe in Section~\ref{sec:max-cut}, such an algorithm would imply a better algorithm for Weighted Max-Cut.

%\anote{Place this somewhere else:}
%An algorithm for $k$-SLR$_2$ that runs in time $N^{(1 - \epsilon)k}$ for some $\epsilon > 0$ will imply a $2^{(1- \epsilon) n}$-time algorithm for $(0, 1)$-CVP$_2$.

% \input{clique-min}
\section{Optimal worst-case hardness of sparse linear regression under the min-weight $k$-clique hypothesis}\label{sec:min-k-clique-slr}

The main result of this section is the theorem below, which shows that there are no {\em better-than-brute-force} algorithms for worst-case sparse linear regression, unless the Min-weight $k$-clique conjecture is false.

% weighted k-clique decision problem
% \anote{Does the Max-weight-$k$-clique conjecture hold for the decision version of the problem as well? If not, we need to change the definitions here.}\vnote{If it holds for optimization (i.e. finding the max weight) it also holds for decision, by binary search. If it holds for search, it also holds for optimization, by removing each vertex in turn and checking if the value of the max-weight clique changes.}

% \anote{Make sure that this runtime is the correct conjecture. Or is it $n^{(1 - \epsilon) k}$-hardness?}
% \vnote{$O(n^k)$ is too strong. $n^{(1-\epsilon)k}$ for any constant $\epsilon>0$ seems like the right conjecture.}

% theorem statement
% \anote{Check running time depending on the $k$-clique conjecture.}
\begin{theorem}\label{thm:weightedclique}
For any integer $k \ge 4$, the $k$-SLR problem in the Euclidean ($\ell_2$) norm requires $\slrcolumns^{k - o(1)}$ (randomized) time, unless the Min-Weight $k$-clique conjecture is false.
\end{theorem}

% high level proof idea

% We want to construct a $k$-SLR instance $(\matA, \vecb)$ such that the non-zero entries in any solution $\vecx^*$ that minimizes $\| \matA x - \vecb \|^2_2$ corresponds to a $k$-clique with maximum weight.

% \paragraph{Proof of Theorem~\ref{thm:weightedclique}.}
\begin{proof}
% [Proof of Theorem~\ref{thm:weightedclique}]
We give a reduction from any Min-weight $k$-clique instance $G = (V, E)$ with positive edge weights to a $k$-SLR$_2$ instance $(\matA, \vecb, \delta)$ as follows. 
First, we set large numbers $\alpha$ and $\beta$, whose significance will be clear in the sequel:
\[\alpha = \sqrt{\max \left\{ 1, \sum_{e \in E} w_e + 8W \right\}}, \quad \beta = \sqrt{\sum_{e \in E} w_e + 8W + \alpha^2 Z} \cdot \max\left\{8 Z, 50 \left(\alpha^2 Z +\sum_{e \in E} w_e \right) \right\},\]
where we define $Z  = \left| \{ (u, v) \notin E \mid u, v \in V \} \right|$ to be the number of non-edges in the graph $G$.
% \anote{Without loss of generality, assume the graph is complete, and assign weight $0$ to non-edges.}
%We enumerate the vertices $u_j \in V$ with index $j \in [n]$,

We enumerate the unordered pairs of vertices $(u_{i_1}, u_{i_2})$ with index $i \in \left[ \binom{\slrcolumns}{2} \right]$ where $u_{i_1}, u_{i_2} \in V$. Let edge $e$ have weight $w_e$.
% high-level description of the matrix A
We first describe how we construct $\matA, \vecb$ which have the following structure:
\[
\matA = 
\begin{pmatrix}
\matC\\
\matD
\end{pmatrix}, \quad
\vecb =
\begin{pmatrix}
\vecc\\
\vecd
\end{pmatrix}
\]

The submatrix $\matC \in \R^{\binom{\slrcolumns}{2} \times \slrcolumns}$ contains one row for each unordered pair of vertices and one column for each vertex in $V$. This submatrix $\matC$, along with the target subvector $\vecc \in \R^{\binom{\slrcolumns}{2}}$, encodes the $k$-clique instance. The submatrix $\matD \in \R^{k \times \slrcolumns}$ and its corresponding target vector $\vecd \in \R^k$ ensure that any solution must have exactly $k$ non-zero entries, all of which are close to $1$.

% non-edges
For each unordered pair of vertices $(u_{i_1}, u_{i_2}) \notin E$ corresponding to a non-edge, we add a row to the matrix $\matC$, and a corresponding entry in the target vector $\vecc$.
\[
\matC_{i, j} =
\begin{dcases}
2 \alpha & \mbox{for } j \in \{i_1, i_2\}\\
0 &\text{otherwise}
\end{dcases}
, \quad
\vecc_i = \alpha
\]
where the large number $\alpha$ is defined above.

% edges
For each $e_i = (u_{i_1}, u_{i_2})$ with weight $w_{e_i} > 0$, we add a row to $\matC$ and a corresponding entry in the target subvector $\vecc$.

\[
\matC_{i, j} =
\begin{dcases}
2 \sqrt{w_{e_i}} & \mbox{for } j \in \{i_1, i_2\}\\
0 &\text{otherwise}
\end{dcases}
, \quad
\vecc_i = \sqrt{w_{e_i}}
\]

% k-sparse 0-1 gadget
As described in Section~\ref{sec:assumptions}, we can assume that we know a partition of the $\slrcolumns$ vertices into $k$ blocks of size $\slrcolumns/k$ such that if a $k$-clique with weight at most $W$ exists, then there is such a $k$-clique with exactly one vertex in each block.
We order the vertices such that the first $\slrcolumns/k$ vertices belong to the first block, the second $\slrcolumns/k$ vertices belong to the second block, and so one.
Let $\vec{1} \in \R^{\slrcolumns/k}$ to be the all $1$'s row vector. Then, we construct $\matD$ and $\vecd$ to each have $k$ rows.

\[
\matD =
\begin{pmatrix}
\beta \vec{1} & \vec{0} & \ldots & \vec{0}\\
\vec{0} & \beta \vec{1} & \ldots & \vec{0}\\
\vec{0} & \vec{0} & \ldots & \beta \vec{1}
\end{pmatrix}, \quad
\vecd =
\begin{pmatrix}
\beta\\
\vdots\\
\beta
\end{pmatrix}
\]

Finally, set $\delta > 0$ such that $\delta = \sqrt{\sum_{e \in E} w_e + 8W + \alpha^2 Z}$.

% completeness
We first prove completeness. In other words, if there is a $k$-clique of weight at most $W$, then we construct a vector $\vecx \in \R^\slrcolumns$ such that $\|\matA \vecx - \vecb \|_2 \le \delta$. We index entries of $\vecx$ by the vertices in $G$ that they correspond to.
Suppose $G$ contains a $k$-clique $\widetilde{K}$ of weight $\widetilde{W} \le W$. Set $x_u = 1$ if vertex $u \in K$, and otherwise set $x_u = 0$.

From this definition, we see that for each non-edge, the contribution of the corresponding row to the error $\|\matA \vecx - \vecb \|^2_2$ is $\alpha^2$, since $\widetilde{K}$ cannot contain both endpoints of any non-edge. For each edge $e \in E$ such that $e$ is not in the clique $\widetilde{K}$, the contribution is $w_e$. For each edge $e$ in $\widetilde{K}$, the contribution is $(4\sqrt{w_e} - \sqrt{w_e})^2 = 9 w_e$.
Since $\widetilde{K}$ is a $k$-clique with one vertex in each of the $k$ blocks, $\|\matD \vecx - \vecd\|_2 = 0$. This gives that 
\begin{align*}
    \|\matA \vecx - \vecb \|^2_2 = \sum_{e \notin \widetilde{K}} w_e + \sum_{e \in \widetilde{K}} 9 w_e + \alpha^2 Z
    & = \sum_{e\in E} w_e + \sum_{e \in \widetilde{K}} 8 w_e + \alpha^2 Z \\
    & = \sum_{e\in E} w_e + 8\widetilde{W} + \alpha^2 Z
    \le \delta^2
\end{align*}
% soundness
Now, we show soundness.
In other words, if there is some $\vecx^*$ such that $\| \matA\vecx^* - \vecb \|_2 \le \delta$, then there must be a $k$-clique in $G$ of weight at most $W$.

% claim that this is a 0-1 k-sparse solution
We show that the only possible such $\vecx^*$ are those that have exactly $k$ non-zero entries, one in each block of $\vecx^*$, with each non-zero entry being very close to $1$.
Suppose for contradiction that one block of $\vecx^*$ does not have any non-zero entries. Then by the choice of $\beta$, we know that $\|\matA \vecx^* - \vecb\|_2 \ge \beta > \delta$. Since $\vecx^*$ is $k$-sparse, this implies that each block has exactly one non-zero entry.
Now, suppose for contradiction that $\vecx^*$ has some non-zero entry $x^*_u$ outside the interval $\left[1 - 2 \delta/\beta, 1 + 2 \delta/\beta \right]$, then
\[\| \matA \vecx^* - \vecb \|_2 \ge |\beta (x^*_u - 1)| \ge 2 \delta > \delta.\]
So all non-zero entries of $\vecx^*$ have to be in $\left[1 - 2 \delta/\beta, 1 + 2 \delta/\beta \right]$.

% claim that this has to be a clique
Now we claim that these $k$ non-zero entries have to correspond to a $k$-clique in the graph. 
Suppose for contradiction that there are two non-zero entries $x^*_u, x^*_v$ such that $(u, v) \notin E$.
Then, counting the error term from rows corresponding to the other $Z - 1$ non-edges as well,

\begin{align*}
    \|\matA \vecx^* - \vecb\|^2_2 &\ge (2 \alpha x^*_u + 2 \alpha x^*_v - \alpha)^2 + \left( 2 \alpha (1 - 2 \delta/\beta) - \alpha \right)^2 (Z - 1)\\
    &\ge \alpha^2 \left( 4 \left( 1 - 2 \delta/\beta \right) - 1 \right)^2 + \alpha^2 \left( 2 (1 - 2 \delta/\beta) - 1 \right)^2 (Z - 1)\\
    &= \alpha^2 \left[ \left(1 - \frac{4 \delta}{\beta} \right)^2 (Z - 1) + \left( 3 - \frac{8 \delta}{\beta} \right)^2 \right]\\
    &\ge \alpha^2 \left[ \left( 9 - \frac{48\delta}{\beta} \right) + \left( 1 - \frac{8 \delta}{\beta} \right)(Z- 1) \right]\\
    &\ge \alpha^2 \left[Z + 8 - \frac{8 \delta Z}{\beta} - \frac{48 \delta }{\beta} \right] \ge \alpha^2 (Z + 6) > \delta^2
\end{align*}
% Setting $\beta > \frac{4\delta}{3 - \sqrt 5}$ and $\alpha^2 > 2 \sum_{e \in E}w_e$. \vnote{First set $\alpha$ and $\beta$ before starting to argue completeness and soundness.}

% claim that this has min weight
Finally, we show that the clique $\widetilde{K}$ so obtained has weight at most $W$. We can give upper and lower bounds to $\widetilde{\delta}^2 \coloneqq \| \matA \vecx^* - \vecb \|^2_2$ as follows. By assumption,
$\widetilde{\delta}^2 \le \delta^2 = \sum_{e \in E} w_e + 8 W + \alpha^2 Z$.

Now, to give a lower bound, consider the contribution to the error $\widetilde{\delta}^2$ from each type of row.
For non-edges, each corresponding row contributes at least $\left( 2\alpha (1 - 2 \delta/\beta)  - \alpha \right)^2$.
For each edge $e \notin \widetilde{K}$, the corresponding row contributes at least $\left( 2 (1 - 2\delta/\beta) - 1 \right)^2 w_e$. 
For each edge $e \in \widetilde{K}$, the corresponding row contributes at least $\left( 4 (1 - 2 \delta/\beta) - 1 \right)^2 w_e$. 
From this, we get
\begin{align*}
    \widetilde{\delta}^2 &\ge \left( 2\alpha (1 - 2 \delta/\beta)  - \alpha \right)^2 Z + \sum_{e \notin \widetilde{K}} \left( 2 (1 - 2\delta/\beta) - 1 \right)^2 w_e + \sum_{e \in \widetilde{K}} \left( 4 (1 - 2 \delta/\beta) - 1 \right)^2 w_e\\
    &= \left( \alpha^2 Z + \sum_{e \notin \widetilde{K}} w_e \right) \left( 1 - \frac{4 \delta}{\beta} \right)^2 + \sum_{e \in \widetilde{K}} w_e \left( 3 - \frac{8 \delta}{\beta} \right)^2\\
    &\ge \left( \alpha^2 Z + \sum_{e \notin \widetilde{K}} w_e \right) \left( 1 - \frac{8 \delta}{\beta} \right) + \sum_{e \in \widetilde{K}} w_e \left( 9 - \frac{48 \delta}{\beta} \right)\\
    &\ge \left( \alpha^2 Z + \sum_{e \in E} w_e \right) \left( 1 - \frac{48 \delta}{\beta} \right) + 8 \sum_{e \notin \widetilde{K}} w_e
\end{align*}

% \[\widetilde{\delta}^2 \ge \left(\frac{1}{2} - \frac{2 \delta}{\beta} \right)^2 \alpha^2 Z + \left( 1 - \frac{2 \delta}{\beta} + \frac{4 \delta^2}{\beta^2}\right) \sum_{e \notin \widetilde K} w_e\]
% The lower bound can be obtained by using the cosine rule.\anote{From properties of the gadget, push Max-cut and/or SETH-hardness of 2-SLR to appendix.}

From these two bounds we get that
\[\frac{48 \delta}{\beta} \left( \alpha^2 Z + \sum_{e \in E} w_e \right) + 8W \ge 8\sum_{e \in \widetilde{K}} w_e \]

Since $\beta \ge 50 \delta \left(\alpha^2 Z + \sum_{e \in E} w_e \right)$, and the edge weights and the threshold $W$ are all integers, we get that
\[W \ge \sum_{e \in \widetilde K} w_e\]

The runtime of this reduction is $O(\slrcolumns^3)$, so if there is an algorithm that solves $k$-sparse linear regression in the $2$-norm in time $\slrcolumns^{k - \epsilon}$ for some $\epsilon > 0$, 
then this reduction gives an algorithm that runs in time $\slrcolumns^3 + \slrcolumns^{k - \epsilon} = O(\slrcolumns^{k - \epsilon})$ for Min-weight $k$-clique when $k \ge 4$.
This completes the proof.
\end{proof}

% The motivation for the circumcircle construction is that for each edge {u, v}, whenever we don't pick both u and v, this will contribute a term of w in the final error. But if the clique contains both, then the error is 0. We are trying to minimize the sum of weights outside the clique, so this maximizes the clique weight.

% We also add rows for non-edges so that we don't end up picking those vertices in the solution. This constrains the solution to have a full k-clique.

% \input{cvp}
\section{Optimal worst-case hardness of sparse linear regression under the hardness of lattice problems}\label{sec:cvp-slr}

%We prove a lower bound for this problem by reducing from the Closest Vector Problem with respect to the $\ell_p$ norm. 

Our main theorem in this section is a fine-grained reduction from the closest vector problem (CVP) in any norm to sparse linear regression (SLR) over the same norm. Using Theorem~\ref{thm:sethcvp}, a corollary is that unless the strong exponential hypothesis (SETH) is false, there is no algorithm for $k$-$\SLR_p$ (defined on odd norms) that runs in time better than $\slrcolumns^{(1-\epsilon) k}$ for any constant $\epsilon>0$.

While CVP and SLR look similar at a high level, our reduction has to bridge a significant difference between the problems: $\SLR$ refers to {\em sparse} solutions over $\mathbb{R}$, whereas CVP refers to {\em unrestricted} solutions over $\mathbb{Z}$.
% theorem statement
\begin{theorem}\label{thm:cvp-slr}
For any integer $k \ge 2$ and $ 1 \le p \le \infty$, there is an efficient reduction from $(0,1)$-$\CVP_p$ on a lattice with rank $\cvpcolumns$ and ambient dimension $\cvprows$ to $k$-$\SLR_p$ on a matrix of dimensions $(\cvprows + k) \times k \cdot 2^{\cvpcolumns/k}$.

It follows that if, for any $\epsilon > 0$, there is a $\slrcolumns^{(1 - \epsilon)k}$-time algorithm for $k$-$\SLR_p$ then there is a $2^{(1 - \epsilon)\cvpcolumns}$-time algorithm for $(0, 1)$-$\CVP_p$.
\end{theorem}

% high level intuition for proof

% proof
\begin{proof}
Given a $(0, 1)$-$\CVP_p$ instance consisting of a lattice basis $\vec B \in \R^{\cvprows \times \cvpcolumns}$, a target vector $\vec{y} \in \R^\cvprows$, and two numbers $r, \tau > 0$, we construct the $k$-$\SLR_p$ instance $(\vec{A}, \vec{b}, \delta)$ as follows.
We set $\delta = r$ and 
partition the columns of $\vec B$ into $k$ groups $\vec{B}_1, \ldots, \vec{B}_k$, each of size at most $\ceil{n/k}$.
\[\vec B =
\begin{pmatrix}
\vec{B}_1 &\ldots &\vec{B}_k
\end{pmatrix}\]

Construct $\vec{A}_i \in \R^{\cvprows \times 2^{\cvpcolumns/k}}$ for $i = 1$ to $k$ such that the columns of $\vec{A}_i$ are formed by summing all possible $2^{\cvpcolumns/k}$ subsets of the columns in $\vec{B}_i$.
Define $\vec{1} \in \R^{1\times 2^{\cvpcolumns/k}}$ to be the row vector of all $1$'s, and let $\alpha \in \R$ be some large number to be determined. %All the other rows are the vector of all zeros.

\[
\matA =
\begin{pmatrix}
\vec{A}_1 & \vec{A}_2 & \ldots &\vec{A}_k\\
\alpha\vec{1} & \vec{0} & \ldots &\vec{0}\\ 
\vec{0} & \alpha\vec{1} & \ldots & \vec{0} \\ 
\vec{0} & \vec{0} & \ldots & \alpha \vec{1}
\end{pmatrix}
, \quad
\vec b =
\begin{pmatrix}
\vec t\\
\alpha\\
\vdots\\
\alpha
\end{pmatrix}
\]

%Here $\vec b \in \R^{d + k}$ is obtained by padding $\vec t$ with $k$ entries of $\alpha$. The construction of $\vec{A}$ from $k$ blocks of the form 
%$\begin{pmatrix}
%\vec{A}_i\\
%\vec{\alpha}_i
%\end{pmatrix}$ 
%allows us to divide the entries of $\vec{x}$ into $k$ corresponding blocks.

Suppose the $(0, 1)$-$\CVP_p$ instance is a \textbf{YES} instance, that is, there is some $\vec{y}^* \in \{0, 1\}^n$ such that $\| \vec B \vec {y}^* - \vec t\|_p \le r$. Then there is a $k$-sparse $\vec{x}^*$ that can be directly computed from $\vec{y}^*$ as follows. 
% \vnote{Something's a little off here. Shouldn't the $\vecx$ entries be divided into $k$ blocks, and the indexing should be $\vecx_{i,S}$ where $i\in [k]$ and $S \subseteq [n/k]$, right?}

\[
x^*_S =
\begin{dcases}
1 &\mbox{if } y^*_i = 1 \text{ for all } i \in S \mbox{ and } y^*_i = 0 \text{ for all } i \notin S \\
0 &\text{ otherwise}
\end{dcases}
\]

Here we have indexed the entries in $\vec x^*$ by the subset $S \subset [(i-1)\cvpcolumns/k + 1, in/k]$ of columns it corresponds to in $\vec B$, where $i$ is the block number. The $\vecx_S^*$ makes the $k$-SLR instance a YES instance since
\[\|\vec A \vecx^* - \vec b\|_p = \| \vec B \vec{y}^* - \vec t\|_p \le r = \delta\]

Now, suppose that the $k$-$\SLR_p$ instance we reduced to is a \textbf{YES} instance. That is, there is some $\vec{x}^*$ such that $\|\vec A \vecx^* - \vec b\|_p \le \delta$. We show that the only possible such solutions are those which have exactly $k$ non-zero entries, one in each block of $\vec{x}^*$, and that these non-zero entries have to be close to $1$.
Suppose for contradiction that one block of $\vec{x}^*$ does not have any non-zero entries, so it is the all-zero vector. Then
$\|\vec A \vecx^* - \vec b\|_p \ge \alpha > \delta$. So by the $k$-sparsity of $\vec{x}^*$, each of its blocks has exactly one non-zero entry.

Now suppose there is some non-zero entry ${x}^*_i$ of $\vec{x}^*$ that is outside the interval $\left[1 - \frac{2 \delta}{\alpha}, 1 + \frac{2 \delta}{\alpha} \right]$. Then
\[\|\vec A \vec{x}^* - \vec b\|_p \ge |\alpha (x^*_i - 1)| \ge 2 \delta > \delta.\]
Therefore we need to consider only the case where each non-zero entry of $\vec{x}^*$ is within the interval $\left[1 - \frac{2 \delta}{\alpha}, 1 + \frac{2 \delta}{\alpha} \right]$.

Consider the vector ${\overline{\vecx}}^*$ obtained by rounding the entries of $\vec{x}^*$ to either $0$ or $1$.
\[
{\overline{x}}^*_i =
\begin{dcases}
1 &{x}^*_i \neq 0\\
0 &{x}^*_i = 0
\end{dcases}
\]

We can construct a solution $\vec{y}^*$ to the $(0, 1)$-CVP$_p$ instance as follows.
\[
{y}^*_i =
\begin{dcases}
1 &\text{there is some }S\text{ such that } \overline{x}^*_S = 1\\
0 &\text{otherwise}
\end{dcases}
\]
% \vnote{See comment about indexing above.}

\begin{align*}
    \|\vec B \vec{y}^* - \vec{t}\|_p &= \| \vec A {\overline{\vecx}}^* - \vec b \|_p\\
    &\le \| \vec{A} {\overline{\vecx}}^* - \vec{A} \vec{x}^* \|_p + \| \vec{A} \vec{x}^* - \vec b \|_p\\
    &\le \| \vec{A} \|_p \cdot k \cdot \frac{2 \delta}{\alpha} + \delta \\ 
    & \leq  \| \vec{B} \|_p \cdot n \cdot \frac{2 \delta}{\alpha} + \delta  \\ 
    & < r + \tau
\end{align*}
for any value of $\tau$, by setting $\alpha = \frac{4 \delta}{\tau} \cdot \|\vec B\|_p \cdot n$ to be large enough.
Here, $\|\vec{B}\|_p := \max_{i\in [\cvpcolumns]} \|\vec{b}_i\|_p$, where $\vec{b}_i$ are the columns of $\matB$.
% \vnote{Before it was $\| \vec{A} \|_p \cdot k^{p-1} \cdot \frac{2 \delta}{\alpha} + \delta$. Changed to make it tighter.}

% The SETH-based reduction in \cite{bennett2017quantitative} gives
% \[\tau = \left( m + \| t^* \|^p_p - 1 \right)^{1/p} - m^{1/p},\]
% where $m$ is the number of clauses in the $k$-SAT instance and $t^*$ is from the parallelepiped construction in \cite{bennett2017quantitative}. 

% \vnote{Compute an explicit bound as a function of $n$ by looking at the SETH-to-CVP reduction.} 
Since $\|\matB \vecy^* - \mathbf{t}\|_p < r+\tau$, by the promise of the $(0,1)$-CVP problem, 
it must be the case that
\[\|\vec B \vec{y}^* - \vec{t}\|_p \le r,\]
making $(\matB,\vect,\delta)$ a \textbf{YES} instance of the CVP problem.
This reduction takes time $O \left( d\cdot k \cdot 2^{\cvpcolumns/k} \right)$, and so if there is an algorithm that solves $k$-SLR$_p$ in time $\slrcolumns^{(1 - \epsilon)k}$, then this reduction gives us an algorithm for $(0,1)$-CVP$_p$ that runs in time $(k \cdot 2^{\cvpcolumns/k})^{(1 - \epsilon)k} = O(2^{(1 - \epsilon)\cvpcolumns})$. This completes the proof.
\end{proof}

The same reduction, combined with the result by Aggarwal et al. \cite{aggarwal2021fine}, gives the following hardness result for sparse linear regression in the $p$-norm with pre-processing.

\begin{corollary}
    If there is an algorithm that solves $k$-SLR$_p$ for odd norms $p$ with preprocessing in time better than $\slrcolumns^{(1 - \epsilon)k}$ for some $\epsilon > 0$, then SETH is false.
\end{corollary}

\section{SETH-hardness of $2$-SLR}\label{sec:seth2slr}

Here we define the Orthogonal Vectors problem (OV), and show that a subquadratic algorithm for $2$-sparse linear regression in the $2$-norm violates the Orthogonal Vectors Hypothesis and, by a result of \cite{williams2004new}, the strong exponential time hypothesis.

Let $d = \omega(n)$. Given two sets $U, V$ of $d$-dimensional vectors with $0,1$ entries, where $|U| = |V| = n$, distinguish between:
\begin{itemize}
    \item \textbf{YES} instances, where there are vectors $\vecu \in U$ and $\vecv \in V$ s.t. $\langle \vecu, \vecv \rangle := \sum_{i=1}^d \vecu[i] \vecv[i] = 0$; and
    \item \textbf{NO} instances, where no such pair of vectors exists.
\end{itemize}

The naive brute force algorithm runs in time $O(n^2 d)$ time, and no algorithm that runs in time $n^{2 - \epsilon} \poly(d)$ is known, for any $\epsilon > 0$. The OV Hypothesis states that this algorithm is essentially optimal, and is implied by both SETH and the Min-Weight $k$-Clique conjecture \cite{williams2004new,abboud2018more}.

% orthogonal vectors conjecture
\begin{definition}[Orthogonal Vectors Hypothesis]
No randomized algorithm can solve OV on instances of size $n$ in
$n^{2 - \epsilon} \poly(d)$ time for constant $\epsilon > 0$.
\end{definition}

The Orthogonal Vectors Hypothesis is implied by the strong exponential time hypothesis, as the following theorem states.

\begin{theorem}[SETH-hardness of Orthogonal Vectors \cite{williams2014finding}] 
If there is an algorithm that solves OV in $n^{2 - \epsilon} \poly(d)$ time, for some $\epsilon > 0$ then SETH is false.
\end{theorem}

With this background, we are ready to state our main result of this section, and its corollary, which guarantees the hardness of $2$-sparse linear regression under SETH or the weaker Orthogonal Vectors Hypothesis.

\begin{proposition}\label{prop:ov-2-slr}
There is an efficient reduction from Orthogonal Vectors with $2$ sets of $n$ vectors from $\{0,1\}^d$ to $2$-sparse linear regression in the $2$-norm, with a $(d + 2) \times 2n$ design matrix.
\end{proposition}

\begin{corollary}
Suppose there is an algorithm that can solve $2$-sparse linear regression in time $n^{2 - \epsilon}\poly(d)$ for some $\epsilon > 0$. Then, the Orthogonal Vectors Hypothesis and SETH are false.
\end{corollary}

\paragraph{Other Results.} In Appendix~\ref{sec:max-cut}, we recall a folklore fine-grained reduction from the max-cut problem on weighted undirected graphs to the CVP problem in the $\ell_2$ norm \cite{aggarwal2021fine}. %
% \ifnum\neurips=0\footnote{We thank Noah Stephens-Davidowitz for pointing out this connection to us~\cite{nsd-personal}.}\fi
In turn, this shows that solving $k$-SLR in $\slrcolumns^{(1-\epsilon)k}$ time implies a corresponding speedup for the weighted max-cut problem, which is believed to be hard~\cite{williams2004new}.

\paragraph{An Open Problem.} It remains open to show that $k$-SLR in the $\ell_2$ norm is optimally hard assuming the strong exponential hypothesis (SETH).

\paragraph{Acknowledgements.} 
The first author would like to thank David Jerison, Frederic Koehler, Kerri Lu and Ankur Moitra for discussions on the problem of sparse linear regression. We thank Noah Stephens-Davidowitz for pointing us to the connection between maxcut and CVP~\cite{aggarwal2021fine} and for his comments on the manuscript.

\bibliographystyle{alpha}
\bibliography{refs}

\appendix

\section{Reduction from Weighted Max-cut to CVP}\label{sec:max-cut}
In this section, we define the weighted max-cut problem, and recall a reduction that shows how any improvement in algorithms for the closest vector problem in the $\ell_2$-norm would imply a better algorithm for weighted max-cut than what is currently known. This reduction, along with Theorem~\ref{thm:cvp-slr} gives further evidence of the $\slrcolumns^k$-hardness of the $k$-sparse linear regression problem in the $\ell_2$-norm.

% \paragraph{Weighted Max-Cut.} 
Suppose we are given a weighted undirected graph $G = (V, E)$ and a weight function $w : E \rightarrow \Z^+$, where edge $e_{i, j} \in E$ between vertices $u_i, u_j \in V$ has weight $w(e_{i, j})$. The maxcut problem asks to find a partition of the vertices so that the cut between them has the largest weight. More precisely, define the quantity
$$ \mathsf{Maxcut}(G) := \max_{\vecx \in \{-1,+1\}^n} \sum_{i,j} (1-x_ix_j)\cdot w(e_{ij})$$

\begin{definition}[Weighted Max-Cut]
Given a weighted undirected graph $G = (V, E)$ on $|V| = n$ vertices, where edge $e \in E$ has weight $w_e \in \Z^+$, and a weight threshold $W$, distinguish between a \textbf{\em YES} instance, where $\mathsf{Maxcut}(G) \ge W$ and a \textbf{\em NO} instance, where $\mathsf{Maxcut}(G) < W$.
\end{definition}

The naive enumeration algorithm runs in time $2^{n}$, and an improvement to even $2^{(1 - \epsilon)n}$ for any $\epsilon > 0$ would be surprising \cite{williams2004new}.

\begin{theorem}\label{thm:max-cut-cvp}
    There is an efficient reduction from the weighted max-cut problem on a graph of $n$ vertices to the $(0,1)$-closest vector problem in the $\ell_2$-norm on an $n$ dimensional lattice.
\end{theorem}

\begin{proof}
Suppose we are given an instance consisting of a graph $G = (V, E)$, and a weight threshold $W$ of \maxcut. Suppose that $G$ has $|V| = n$ vertices and $|E| = m$ edges. We index edges $e_i \in E$ with the index $i \in [m]$. We index vertices $u_j \in V$ with $j \in [n]$.
We construct an instance $(\matB, \vect, r, \tau)$ of $(0,1)$-CVP$_2$ as follows.

\[\matB =
\begin{pmatrix}
\matC\\
\matD
\end{pmatrix},
\quad
\vect =
\begin{pmatrix}
\vecc\\
\vecd
\end{pmatrix},
\quad
r^2 = \sum_{e \in E} w_e - W + n \alpha^2,
\quad
\tau = \sqrt{r^2 + 1} - r
\]

%%% construct matrix C
For each edge $e_i = (u_{j_1}, u_{j_2}) \in E$ with edge $w_{e_i}$, we add a row (the $i$th row) to the matrix $\matB_1 \in \R^{m \times n}$. The entries of the $i$th row, and the corresponding entry in the target vector $\vecc$ is defined by
\[
\matC_{i,j} =
\begin{dcases}
\sqrt{w_{e_i}} &j \in \{j_1, j_2\}\\
0 &\text{otherwise}
\end{dcases},
\quad
\vecc_i =\sqrt{w_{e_i}}.
\]

Also, let $\matD = 2 \alpha I_n \in \R^{n \times n}$, where $\alpha = \sqrt{2 \sum_{e \in E} w_e}$ and $I_n$ is the $n \times n$ identity matrix. Let $\vecd = \alpha \vec{1} \in \R^n$ be the all ones column vector of dimension $n$.

%%% 
First, we show completeness. Suppose there is a cut $V_1 \cup V_2 = V$ of weight at least $W$, and define the set of edges in the cut as $C = \{e = (u,v)\in E \mid u\in V_1, v\in V_2 \text{ or } u \in V_2, v \in V_1\}$. Then, construct the vector $\vecy$ as follows.
\[
y_j =
\begin{dcases}
1 & u_j \in V_1\\
0 & u_j \in V_2
\end{dcases}
\]
Then
\begin{align*}
    \|\matB \vecy - \vect\|^2_2 &= \sum_{e = (u, v) \notin C} w_e + n \alpha^2\\
    &\le \sum_{e \in E} w_e - W + n \alpha^2 = r^2.
\end{align*}

Further, because of the integrality condition on the weights of the graph edges, it is true for any vector $\vecy$ that either $\| \matB \vecy - \vect\|_2 \le r$ or $\| \matB \vecy - \vect\|_2 \ge r + \tau$.

Now, we show soundness. Suppose there is some vector $\vecy^* \in \Z^n$ such that $\| \matB \vecy^* - \vect\|_2 \le r$. We show that this implies $\mathsf{Maxcut}(G) \ge W$.
Setting $\alpha^2 = 2 \sum_{e \in E} w_e$, we see that only solutions $\vecy^* \in \{0,1\}^n$ are possible. Then define a partition on the vertex set $V = V^*_1 \cup V^*_2$ such that $V^*_1 = \{ u_j \mid y^*_j = 1\}$ and $V^*_2 = V \setminus V^*_1 = \{u_j \mid y^*_j = 0\}$. Let $C^* = \{e = (u,v)\in E \mid u\in V^*_1, v\in V^*_2 \text{ or } u \in V^*_2, v \in V^*_1\}$ be the cut defined by this partition. Let the weight of the cut $C^*$ be $W^*$. Then, we see that

\begin{align*}
    \|\matB \vecy^* - \vect\|^2_2 &= \sum_{e = (u, v) \notin C^*} w_e + n \alpha^2 = \sum_{e \in E} w_e - W^* + n \alpha^2
    \le r^2,
\end{align*}
and so $W^* \ge W$.

This reduction runs in time $O\left(n(m + n)\right) \le O(n^3)$, and this completes the proof.
\end{proof}

\begin{corollary}
If there is some $\epsilon > 0$ such that $(0,1)$-CVP in the $\ell_2$ can be solved in time $2^{(1 - \epsilon)n}$, then there is an algorithm that solves weighted max-cut in time $2^{(1 - \epsilon)n}$.
\end{corollary}

\section{Proof for Section~\ref{sec:seth2slr}}
\begin{proof}[Proof for Proposition~\ref{prop:ov-2-slr}]
We give a reduction from OV to $2$-sparse linear regression in the $2$-norm.
Suppose we are given an OV instance consisting of two sets $U, V$ of $d$-dimensional vectors $\vecu_1, \ldots, \vecu_n \in U$ and $\vecv_1, \ldots, \vecv_n \in V$, each of size $n$.
Define the normalized vectors $\vecu'_i = \frac{1}{\|\vecu_i \|} \vecu_i$ and 
$\vecv'_i = \frac{1}{\|\vecv_i \|} \vecv_i$ such that $\| \vecu'_i \|_2 = \| \vecv'_i \|_2 = 1$ for all $i \in \{1, \ldots, n\}$.
Then construct an instance $(\matA, \vecb, \delta)$ of $2$-sparse linear regression as follows. 
\[
\matA =
\begin{pmatrix}
| & &| &| & &|\\
\vecu'_1 &\ldots &\vecu'_n &\vecv'_1 &\ldots &\vecv'_n\\
| & &| &| & &|\\
\alpha &\ldots &\alpha &\alpha &\ldots &\alpha\\
0 &\ldots &0 &0 &\ldots &0
\end{pmatrix}, 
\quad
b =
\begin{pmatrix}
0\\
\vdots\\
0\\
\alpha\\
\alpha
\end{pmatrix},
\quad
\delta = \sqrt{2},
\]
where we set $\alpha = 16 \delta d$, whose significance will become clear later.

We first show completeness. Suppose there are two vectors $\vecu_i \in U, \vecv_j \in V$ such that $\left \langle \vecu_i, \vecv_j \right \rangle = \left\langle \vecu'_i, \vecv'_j \right\rangle = 0$. Then, construct the vector $\vecx \in \R^{2n}$ such that
\[
x_\ell =
\begin{dcases}
1 &\ell \in \{i, n + j\}\\
0 &\text{otherwise}
\end{dcases}.
\]
Then $\| \matA \vecx - \vecb \|^2_2 = \| \vecu'_i + \vecv'_j \|^2_2 = \| \vecu'_i \|^2_2 + \| \vecv'_j \|^2_2 + 2 \left\langle \vecu'_i, \vecv'_j \right\rangle = 2 = \delta^2$.

We now show soundness. Suppose that there is some vector $\vecx \in \R^{2n}$ such that $\| \matA \vecx - \vecb \|_2 \le \delta$.
First, we claim that $\vecx$ must have at least one non-zero entry indexed from the index set $I_1 := \{1, \ldots, n\}$, and at least one non-zero entry indexed from $I_2 := \{n+1 \ldots, 2n\}$.
Suppose for contradiction that $x_i = 0$ for all $i \in I_1$. Then, $\| \matA \vecx - \vecb \| \ge \alpha > \delta$. Using a similar argument for $I_2$, and because $\vecx$ is $2$-sparse, we know that $\vecx$ must have exactly two non-zero entries, one indexed from each of $I_1$ and $I_2$. 
Let the two non-zero entries be $x_i$ and $x_{j}$, with $i \in I_1$ and $j \in I_2$.

Next, we claim that the two non-zero entries have to be close to $1$. Specifically, $x_i, x_j \in \left[ 1- 2 \delta/ \alpha, 1 + 2 \delta/\alpha \right]$. Otherwise, we get a contradiction since $\|\matA \vecx - \vecb\| \ge 2 \delta$.

Note that since the vectors of the OV instance $\vecu \in U, \vecv \in V$ have $0,1$ entries, we know that either $\langle \vecu, \vecv \rangle = 0$ or $\langle \vecu, \vecv \rangle \ge 1$. Translating this to a statement about the normalized vectors, we know that either $\left\langle \vecu'_i, \vecv'_j \right\rangle = 0$ or $\left\langle \vecu'_i, \vecv'_j \right\rangle \ge 1/d$.

Now, we can lower and upper bound the prediction error $\| \matA \vecx - \vecb \|$ as follows. By assumption, we know that $\| \matA \vecx - \vecb \|^2 \le \delta^2 = 2$. Further,
\begin{align*}
    \| \matA \vecx - \vecb \|^2 &\ge \| x_i \vecu'_i + x_j \vecv'_j \|^2 = x^2_i \| \vecu'_i \|^2 + x^2_j \| \vecv'_j \|^2 + 2 x_i x_j \left\langle \vecu'_i, \vecv'_j \right\rangle = x_i^2 + x^2_j + 2 x_i x_j \left\langle \vecu'_i, \vecv'_j \right\rangle\\
    &\ge 2 \left( 1 - \frac{2 \delta}{\alpha} \right)^2 + 2  \left( 1 - \frac{2 \delta}{\alpha} \right)^2 \left\langle \vecu'_i, \vecv'_j \right\rangle \ge 2 \left( 1 - \frac{4 \delta}{\alpha} \right) + 2 \left( 1 - \frac{4 \delta}{\alpha} \right) \left\langle \vecu'_i, \vecv'_j \right\rangle\\
    &\ge 2 \left( 1 - \frac{4 \delta}{\alpha} \right) + \left\langle \vecu'_i, \vecv'_j \right\rangle.
\end{align*}

Combining the upper and lower bounds from above, $\left\langle \vecu'_i, \vecv'_j \right\rangle \le \frac{8 \delta}{\alpha} = \frac{1}{2d} < 1/d$. So, it must be true that $\left\langle \vecu'_i, \vecv'_j \right\rangle = 0$, and so there is an orthogonal pair of vectors.

This reduction runs in time $O(n d)$ time, so an algorithm that solves $2$-sparse linear regression in time $O(N^{2 - \epsilon})$ for some $\epsilon > 0$ will solve OV in time $O(n^{2 - \epsilon})$. This completes the proof.
\end{proof}

\end{document}